\documentclass{article}

\PassOptionsToPackage{numbers, compress}{natbib}

\usepackage[preprint]{neurips_2026}
\newtheorem{proposition}{Proposition}
\newtheorem{proof}{Proof}
\usepackage{pifont}
\newcommand{\cmark}{\ding{51}}
\newcommand{\xmark}{\ding{55}}
\usepackage{amsmath,amssymb}
\usepackage{graphicx}
\usepackage[utf8]{inputenc} 
\usepackage[T1]{fontenc}    
\usepackage{hyperref}       
\usepackage{url}            
\usepackage{booktabs}       
\usepackage{amsfonts}       
\usepackage{nicefrac}       
\usepackage{microtype}      
\usepackage{xcolor}         

\title{Beyond World-Frame Action Heads: Motion-Centric Action Frames for Vision-Language-Action Models}

\author{%
  \textbf{Huoren Yang$^{1,2}$ 
  Jianchao Zhao$^{1,2}$ 
  Hu Yusong$^{2}$ 
  Qiguan Ou$^{2}$ 
  Yuyang Gao$^{2}$} \\
  \textbf{Wei Ke$^{1}$ 
  Yuhang He$^{1}$ 
  SongLin Dong$^{3}$ 
  Zhiheng Ma$^{3}$ 
  Yihong Gong$^{1}$} \\
  \\
  $^{1}$Xi'an Jiaotong University, Xi'an, Shaanxi, CN \\
  $^{2}$One Robotics, Shenzhen, Guangdong, China \\
  $^{3}$Shenzhen University of Advanced Technology, Shenzhen, Guangdong, CN
}

\begin{document}

\maketitle

\begin{abstract}
Vision-Language-Action (VLA) models have advanced rapidly with stronger backbones, broader pre-training, and larger demonstration datasets, yet their action heads remain largely homogeneous: most directly predict action commands in a fixed world coordinate frame. We propose \textbf{MCF-Proto}, a lightweight action head that equips VLA policies with a Motion-Centric Action Frame (MCF) and a prototype-based action parameterization. At each step, the policy predicts a rotation $R_t \in SO(3)$, composes actions in the transformed local frame from a set of prototypes, and maps them back to the world frame for end-to-end training, using only standard demonstrations without auxiliary supervision.
This simple design induces stable emergent structure. Without explicit directional labels, the learned local frames develop a stable geometric structure whose axes are strongly compatible with demonstrated end-effector motion. Meanwhile, actions in the learned representation become substantially more compact, with variation captured by fewer dominant directions and more regularly organized by shared prototypes. These structural properties translate into improved robustness, especially under geometric perturbations. Our results suggest that adding lightweight geometric and compositional structure to the action head can materially improve how VLA policies organize and generalize robotic manipulation behavior. An anonymized code repository is provided in the supplementary material.
\end{abstract}

\section{Introduction}
\label{sec:intro}

Vision-Language-Action (VLA) models have become a dominant paradigm for robotic manipulation by combining large-scale vision-language pre-training with imitation learning on broad robot datasets~\cite{RT-1,Rt-2,Open-x-embodiment,Octo,openvla,3robocat}. Recent progress has been driven largely by stronger perceptual backbones, improved language grounding, and more diverse demonstrations~\cite{vima,unleashing,pact,shridhar2023perceiver}. In contrast, the action side of VLA models remains strikingly uniform: most policies still use a standard regression head to predict end-effector commands in a fixed global world coordinate frame~\cite{shridhar2023perceiver,chi2025diffusion,florence2022implicit,zhao2023learning}.

While simple and general-purpose, this design does not obviously match the structure of manipulation. Many manipulation behaviors are inherently local and geometric~\cite{khatib1987unified,calinon2016tutorial,calinon2014task,simeonov2022neural,ryu2022equivariant}. Actions such as approaching a handle, opening a drawer, inserting along a slot, or rotating around an object-relative axis may share similar interaction structure, yet appear as very different vectors in the world frame. Under such a parameterization, the policy must absorb these variations implicitly in its latent representation while decoding them into a single fixed output space. This motivates a natural question: can VLA policies organize actions more effectively if equipped with a learned local frame and a structured mechanism for action composition within that frame?
\begin{figure}[t]
  \centering
  \includegraphics[width=1.0\textwidth]{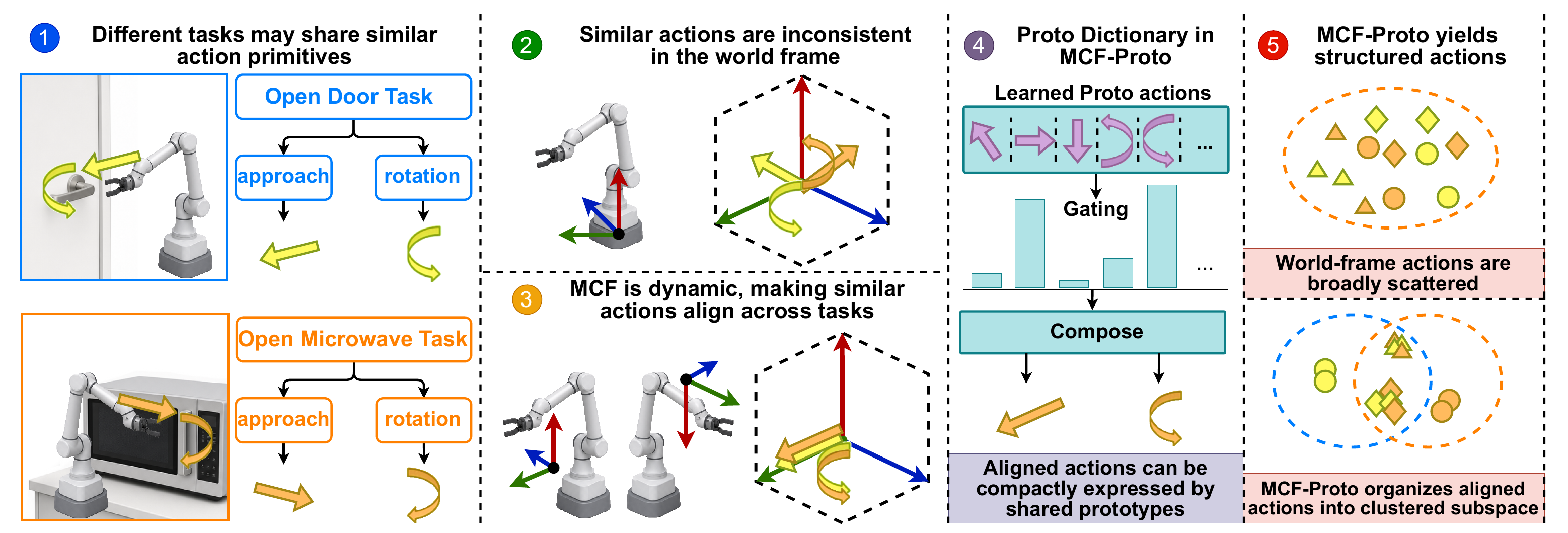}
  \caption{Motivation of MCF-Proto. Similar manipulation behaviors can appear very different when expressed in a fixed world frame, but become better aligned in the learned Motion-Centric Frame (MCF). By combining this local frame with prototype-based action composition, MCF-Proto produces a more compact and structured action representation that is easier to share across tasks and transfer under geometric variation.
}
\vspace{-5pt}
  \label{fig:intro}   
\end{figure}

To answer this question, we propose \textbf{MCF-Proto}, a lightweight action head redesign for standard VLA policies. As shown in Figure~\ref{fig:intro}, instead of directly regressing actions in the world frame, the policy first predicts an observation-conditioned rotation \(R_t \in SO(3)\) to define a Motion-Centric Action Frame (MCF). Within this frame, actions are generated as soft combinations of a small shared prototype dictionary, encouraged by orthogonal regularization to remain distinct and non-redundant, and are then mapped back to the world frame for supervision. The entire model is trained end-to-end using only standard demonstrations, without auxiliary labels such as object poses, contact states, or motion directions.

This simple output-side modification induces consistent emergent structure. Without explicit directional supervision, the learned local frames develop a stable directional structure whose axes remain strongly compatible with demonstrated end-effector motion across diverse tasks. At the same time, actions expressed in the learned representation become substantially more compact: their variation concentrates on fewer dominant directions and is more effectively organized by the shared prototypes. These structural changes translate into better empirical performance. On LIBERO~\cite{liu2023libero}, MCF-Proto improves over strong VLA baselines, and on LIBERO-plus, it yields especially strong gains under geometric perturbations such as robot initial pose changes and camera shifts.

Our contributions are threefold:
\begin{itemize}
    \item We introduce MCF-Proto, a lightweight VLA action head that combines a Motion-Centric Action Frame, and a shared action prototype dictionary, while requiring only standard action supervision.
    \item We show that this design induces stable emergent structure: the learned frame provides a local basis whose axes remain closely compatible with demonstrated end-effector motion, and the resulting action representation becomes more compact and better suited to prototype-based composition.
    \item We demonstrate that these structural properties improve both task performance and robustness, with particularly strong gains under geometric perturbations on LIBERO-plus.
\end{itemize}
\section{Related works}

\paragraph{Behavioral Cloning.}
Behavioral cloning (BC) has become a strong, scalable paradigm for visuomotor manipulation due to simplicity, stable optimization, and compatibility with large-scale demonstration data. Recent works show that image-conditioned policies trained with larger datasets, stronger visual encoders, and multi-task formulations can produce increasingly capable manipulation behaviors across diverse settings~\cite{jang2022bc_Z,shridhar2022cliport,shridhar2023perceiver,chi2025diffusion,zhao2023learning,goyal2023rvt,chen2023polarnet}. These advances have substantially improved perception, policy capacity, and task coverage, establishing BC-style visuomotor policies as competitive alternatives to more complex reinforcement learning pipelines in data-rich regimes~\cite{jang2022bc_Z,shridhar2023perceiver,chi2025diffusion,zhao2023learning,goyal2023rvt,chen2023polarnet,Octo}. Meanwhile, much of this progress has focused on better observation modeling and larger policy backbones, while the action prediction layer often remains relatively generic, directly regressing end-effector commands in a fixed world or robot frame~\cite{RT-1,Open-x-embodiment,chi2025diffusion,florence2022implicit}. Our work is motivated by the view that action parameterization itself is an important design choice: even with a strong backbone, expressing manipulation actions in a more structured local form can make action patterns easier to share across scenes and improve robustness to geometric variation.

\paragraph{Action Frames.}
A long line of work in robotics has emphasized that manipulation actions are often more naturally expressed in local, object-centric, contact-centric, or task-aligned coordinate systems rather than a single global frame~\cite{khatib1987unified,calinon2016tutorial,calinon2014task,alizadeh2014learning,chen2022manipulation}. Such representations can simplify control, improve invariance to changes in scene layout, and better reflect the geometry of physical interaction~\cite{nakanishi2008operational,calinon2014task,simeonov2022neural,Eisner_se3}. Prior approaches have incorporated local frames through object-centric policies, affordance-guided control, equivariant architectures, or explicit geometric reasoning modules~\cite{devin2018deep,chen2022manipulation,sharma2020learning,simeonov2022neural,ryu2022equivariant,simeonov2023se,Eisner_se3,zhu2023learning,shen2023distilled}. More broadly, these works suggest that choosing an appropriate coordinate system can substantially reduce the burden on downstream control prediction. Our work follows this perspective, but integrates it into a lightweight policy-head design for imitation learning: instead of assuming a predefined object or interaction frame, we predict an observation-conditioned MCF directly from policy features, and use it to express the action before mapping it back to the world frame. In this way, the learned frame serves as a compact intermediate representation for action generation rather than an externally specified geometric input.

\paragraph{Prototype Actions.}
Our action representation is also related to prior work on motion primitives, dictionary learning, structured latent actions, and compositional policy outputs~\cite{ijspeert2013dynamical,paraschos2013probabilistic,morrow1997manipulation,felip2013manipulation}. Across robotics and machine learning, reusable basis elements or prototype components have been used to capture recurring structure in trajectories, skills, and control policies, often improving parameter efficiency, interpretability, and transfer~\cite{ijspeert2013dynamical,paraschos2013probabilistic,ruckert2013learned,przystupa2023deep,frank2021constrained}. Related ideas also appear in mixture-of-experts models, sparse coding, and other compositional prediction frameworks, where outputs are formed from a small set of shared components~\cite{shafiullah2022behavior,cui2023from,lee2024behavior}. Our work adopts this general principle at the action-head level: instead of predicting each action dimension independently, we compose local translational and rotational actions from small sets of shared prototypes, each implemented as a learned linear map to a 3D action increment. The key difference is that this composition is performed in the learned MCF rather than directly in the world frame. This combination allows similar motion patterns to be shared after geometric alignment, leading to a more compact local action representation.
\section{Method}
\label{sec:method}
\begin{figure}[t]
  \centering
  \includegraphics[width=1.0\textwidth]{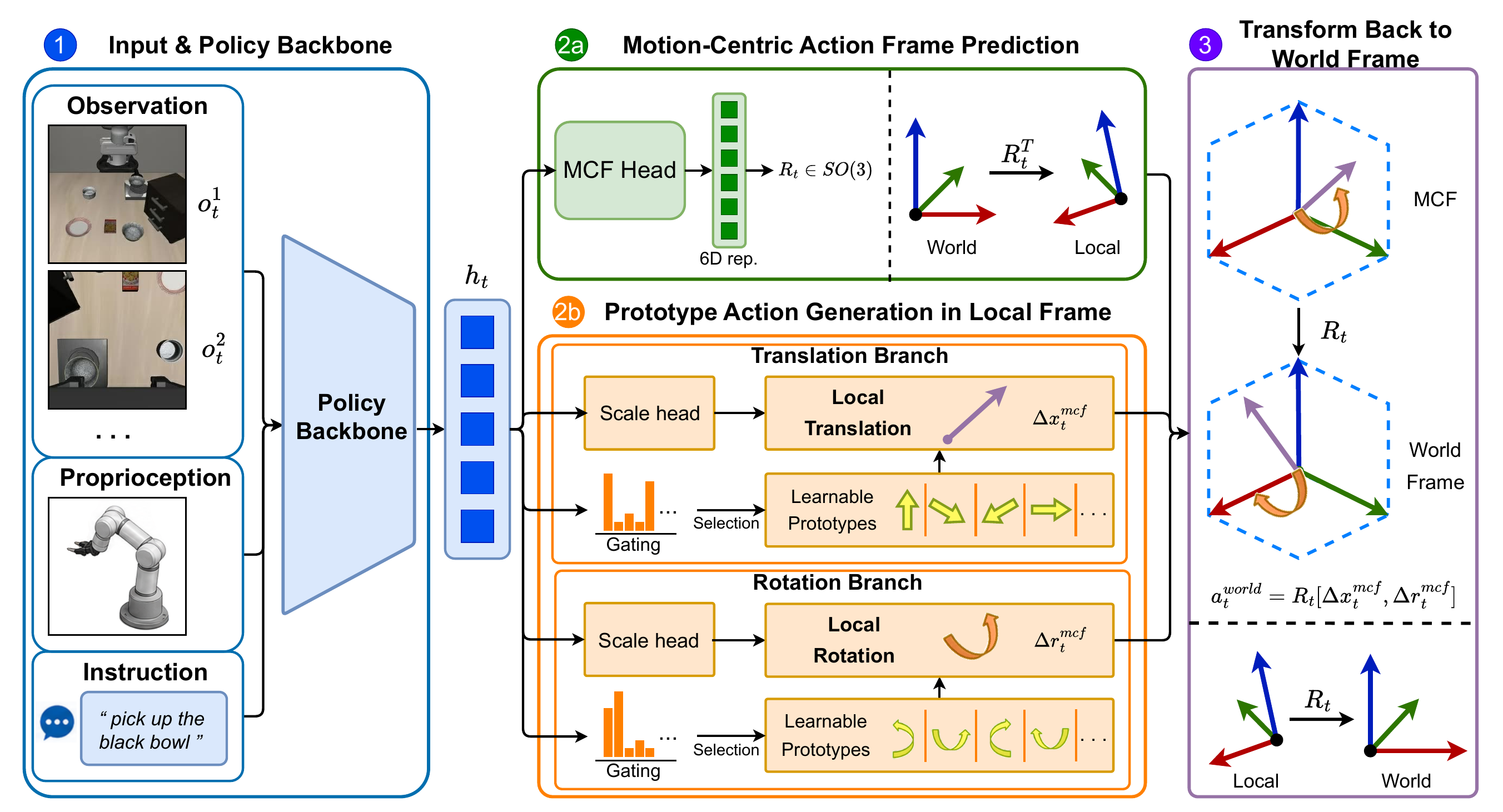}
  \caption{Overview of MCF-Proto. Given observations and task instruction, the policy backbone extracts latent features that are used to predict MCF. Translation and rotation actions are then generated in this local frame through soft combinations of shared prototypes, and finally transformed back to the world frame for standard action supervision.}
  \label{fig:method}   

\end{figure}
We introduce \textbf{MCF-Proto}, a lightweight redesign of the action head for vision-language-action policies. Instead of predicting end-effector actions directly in the world frame, MCF-Proto first predicts an observation-conditioned local rotation \(R_t \in SO(3)\), uses it to define a Motion-Centric Action Frame (MCF), composes actions in that frame from a small set of shared prototypes, and then maps the result back to the world frame for standard imitation learning.

This design is motivated by two simple observations. First, manipulation actions are often easier to describe in a local geometric frame than in a fixed global one. Second, many manipulation behaviors share reusable local motion patterns even when their world-frame realizations vary across scenes, object poses, or viewpoints. MCF-Proto builds these two biases directly into the policy output layer while leaving the backbone and supervision unchanged.

\subsection{Setup}
\label{subsec:setup}

Given an observation \(o_t\), a VLA backbone produces a latent feature
$h_t = f_{\text{backbone}}(o_t).$
We consider actions consisting of end-effector translation and rotation increments together with any remaining control dimensions:
$a_t^\star = (\Delta x_t^\star, \Delta r_t^\star, a_{t,\mathrm{rest}}^\star),$
where \(\Delta x_t^\star \in \mathbb{R}^3\) and \(\Delta r_t^\star \in \mathbb{R}^3\). A conventional action head would predict these quantities directly in the world frame. Our goal is to predict the same action target through a structured local representation.

\subsection{Motion-Centric Action Frame}
\label{subsec:mcf}

As shown in Figure~\ref{fig:method}, from the policy feature \(h_t\), we predict a rotation \(R_t \in SO(3)\) using a lightweight MLP with a continuous 6D rotation parameterization. This rotation defines the local \textit{Motion-Centric Action Frame} at time \(t\).

Let \(\Delta x_t^{\mathrm{mcf}} \in \mathbb{R}^3\) and \(\Delta r_t^{\mathrm{mcf}} \in \mathbb{R}^3\) denote the translational and rotational increments expressed in the MCF. They are mapped back to the world frame by:
\[
\Delta x_t = R_t \Delta x_t^{\mathrm{mcf}}, \qquad
\Delta r_t = R_t \Delta r_t^{\mathrm{mcf}}.
\]
The final predicted action is
$\hat{a}_t = (\Delta x_t, \Delta r_t, \hat{a}_{t,\mathrm{rest}}),$
where \(\hat{a}_{t,\mathrm{rest}}\) is predicted by a small auxiliary head.

The role of the MCF is to absorb part of the scene-dependent geometric variation before action decoding. As we show later, the learned frame develops a smoothly varying directional structure whose axes are often compatible with demonstrated end-effector motion.

\subsection{Prototype composition in the local frame}
\label{subsec:proto}

Within the MCF, we parameterize actions using a small set of shared prototypes. Concretely, each prototype is a matrix \(B_k \in \mathbb{R}^{3 \times d}\)
that maps a latent coefficient vector to a 3D local action increment. We maintain separate prototype dictionaries for translation and rotation:
\[
B^{\text{trans}} = \{B_k^{\text{trans}}\}_{k=1}^{K_t}, \qquad
B^{\text{rot}} = \{B_k^{\text{rot}}\}_{k=1}^{K_r},
\]
Given the latent state \(h_t\), the model predicts latent scaling vectors
$z^{\text{trans}}, z^{\text{rot}} \in \mathbb{R}^d,$
as well as prototype gating logits for translation and rotation. After a softmax, these produce gating distributions \(\pi_t^{x}\) and \(\pi_t^{r}\), respectively. The local action is then composed as:
\[
\Delta x_t^{\mathrm{mcf}} = \sum_{k=1}^{K_t} \pi_t^{x}[k]\, B_k^{\text{trans}} z^{\text{trans}},
\qquad
\Delta r_t^{\mathrm{mcf}} = \sum_{k=1}^{K_r} \pi_t^{r}[k]\, B_k^{\text{rot}} z^{\text{rot}}.
\]

This local prototype parameterization encourages the policy to express actions through a small set of shared prototype mappings modulated by latent scale coefficients. Combined with the learned local frame, it allows similar manipulation behaviors to exhibit shared structure even when their world-frame directions differ. In practice, this leads to a substantially more concentrated action distribution in the learned local frame than in the original world frame.

The two components of MCF-Proto play distinct but complementary roles. The predicted local frame reduces scene-dependent geometric variation by aligning actions into an observation-conditioned coordinate system, while the prototype dictionaries exploit the resulting alignment to capture recurrent motion patterns through a shared low-dimensional parameterization. Without the local frame, similar behaviors may correspond to very different world-frame directions and are therefore harder to organize consistently across tasks. Without prototype composition, the local frame alone provides geometric alignment but does not explicitly encourage shared structure or low-dimensional organization in the action output space. Their combination yields a more structured representation in which directional alignment and prototype-based composition emerge jointly.

To keep the prototype dictionaries well-conditioned, we apply an orthogonality regularizer:
\[
\mathcal{L}_{\text{ortho}}
=
\left\|
(B^{\text{trans}})^\top B^{\text{trans}} - I
\right\|_F^2
+
\left\|
(B^{\text{rot}})^\top B^{\text{rot}} - I
\right\|_F^2.
\]

\subsection{Learning objective}
\label{subsec:objective}

MCF-Proto is trained end-to-end by supervising the final predicted end-effector incremental action in the world frame. Let \(a_t^\star\) denote the demonstrated action and let \(\hat{a}_t\) denote the predicted action. We supervise translational and gripper dimensions using an \(L_1\) loss, and rotational dimensions using a Smooth \(L_1\) loss in radians. Let \(I_{\text{trans}}, I_{\text{rot}}, I_{\text{grip}}\) denote the translation, rotation, and gripper index sets in the full action vector. The action supervision term is:
\[
\mathcal{L}_{\text{act}}
=
\|\hat{a}_t[I_{\text{trans}}] - a_t^\star[I_{\text{trans}}]\|_1
+
\|\hat{a}_t[I_{\text{grip}}] - a_t^\star[I_{\text{grip}}]\|_1
+
\mathrm{SmoothL1}\!\left(\hat{a}_t[I_{\text{rot}}], a_t^\star[I_{\text{rot}}]\right).
\]

To encourage temporal consistency of the learned motion-centric frame, we regularize consecutive frame rotations to vary smoothly over time. Intuitively, the local frame is encouraged to evolve smoothly over time rather than rotating abruptly to follow every instantaneous action change. When the action direction shifts across task stages, the model can explain such variation through different axes within the frame, while preserving a relatively stable local basis. Let \(R_t\) denote the frame rotation at time \(t\). We define
$\Delta R_t = R_t R_{t-1}^\top,
\cos \theta_t = \frac{\mathrm{tr}(\Delta R_t) - 1}{2},$
and use the smoothness loss
$\mathcal{L}_{\text{smooth}} = 1 - \cos \theta_t.$
The final objective is:
\[
\mathcal{L}
=
\mathcal{L}_{\text{act}}
+
\lambda_{\text{ortho}} \mathcal{L}_{\text{ortho}}
+
\lambda_{\text{smooth}} \mathcal{L}_{\text{smooth}}.
\]

\section{Experiments}
\label{sec:experiments}

We evaluate MCF-Proto from four perspectives: benchmark performance, robustness under distribution shift, analysis of the learned local action representation, and real-world validation. We first report results on LIBERO~\cite{liu2023libero} and LIBERO-plus~\cite{fei2025liberoplus}, then analyze how the learned motion-centric frame reshapes the action space through concentration and geometric compatibility diagnostics, and finally present real-world results and hierarchical ablations.

\subsection{Benchmarks and metrics}
\label{subsec:benchmarks}

\textbf{LIBERO.}
We evaluate on the standard LIBERO benchmark, including the Spatial, Object, Goal, and Long suites. We report task success rate (\%) on each suite and the average across all suites.
\textbf{LIBERO-plus.}
We further evaluate robustness on LIBERO-plus, which introduces seven perturbation categories: Camera, Robot, Language, Light, Background, Noise, and Layout. We report success rate (\%) for each category.

\subsection{Baselines}
\label{subsec:baselines}

We compare against a diverse set of recent VLA and robot policy baselines. On LIBERO, we include $\pi_0$+FAST~\cite{fast}, OpenVLA-OFT~\cite{openvlaoft}, $\pi_0$~\cite{pi_0}, FLOWER~\cite{reuss2025flower}, GR00T-N1.5~\cite{gr00tn1_2025}, and BEAST~\cite{zhou2025beastefficienttokenizationbsplines}. On LIBERO-plus, we compare against OpenVLA~\cite{openvla}, OpenVLA-OFT, $\pi_0$, $\pi_0$-fast, Nora~\cite{nora}, WorldVLA~\cite{WorldVLA}, UniVLA~\cite{univla}, and RIPT-VLA~\cite{riptvla}.
\begin{table}[htbp]
  \centering
  \small
  \caption{Experimental Results on the LIBERO Benchmark. Success rate (\%) is reported for each task suite. 
  \textbf{Bold} indicates the best performance, \underline{underlined} indicates the second best.}
  \label{tab:libero_main}
  \begin{tabular}{lccccc}
    \toprule
    Model               & Spatial & Object & Goal & Long & Avg \\
    \midrule
    $\pi_0$+FAST        & 96.4    & 96.8   & 88.6 & 60.2 & 85.5 \\
    OpenVLA-OFT         & 97.6    & 98.4   & 97.9 & 94.5 & 97.1 \\
    $\pi_0$             & 96.8    & 98.8   & 95.8 & 85.2 & 94.1 \\
    FLOWER              & 97.1    & 96.7   & 95.6 & 93.5 & 95.7 \\
    GR00T-N1.5          & 92.0    & 92.0   & 86.0 & 76.0 & 86.5 \\
    BEAST               & 92.9    & 97.5   & 93.1 & 86.4 & 92.5 \\
    \midrule
    Ours                & \textbf{98.6$\pm$0.2} & \textbf{98.8$\pm$0.2} & \textbf{98.4$\pm$0.2} & \textbf{95.0$\pm$0.5} & \textbf{97.7$\pm$0.3} \\
    \bottomrule
  \end{tabular}
\end{table}

\begin{table}[htbp]
  \centering
  \small
  \caption{Experimental Results on the LIBERO-plus Benchmark. Success rate (\%) is reported for each task suite. Best is marked in \textbf{bold}, second best is marked in \underline{underlined}.}
  \label{tab:libero_plus_main}
  \setlength{\tabcolsep}{5pt}
  \begin{tabular}{lccccccc}
    \toprule
    Model          & Camera & Robot & Language & Light & Background & Noise & Layout \\
    \midrule
    OpenVLA             & 1.1    & 4.1   & 26.8     & 4.4   & 25.3       & 19.3  & 31.6   \\
    OpenVLA-OFT         & 59.7   & 37.2  & \textbf{81.5}     & 85.8  & \underline{92.4}       & 76.7  & \underline{77.1}   \\
    $\pi_0$             & 15.8   & 6.6   & 61.0     & 79.6  & 78.5       & 79.4  & 70.4   \\
    $\pi_0$-fast        & 66.4   & 24.8  & 63.3     & 73.0  & 67.7       & 75.8  & 70.3   \\
    Nora                & 4.0    & 41.1  & 67.0     & 31.0  & 50.5       & 17.6  & 63.9   \\
    WorldVLA            & 0.3    & 30.2  & 44.2     & 29.4  & 14.5       & 12.2  & 39.4   \\
    UniVLA              & 4.3    & \underline{50.3}  & 71.8     & 59.1  & 80.0       & 25.3  & 34.3   \\
    RIPT-VLA            & \underline{58.3} & 36.7  & \underline{80.1} & \underline{87.9} & 90.4 & \underline{73.8} & 76.5 \\
    \midrule
    Ours                & \textbf{69.7$\pm$0.2} & \textbf{66.0$\pm$0.2} & \underline{80.1$\pm$0.3} & \textbf{98.3$\pm$0.1} & \textbf{93.3$\pm$0.2} & \textbf{84.0$\pm$0.3} & \textbf{82.1$\pm$0.2} \\
    \bottomrule
  \end{tabular}
\end{table}
\subsection{Main results on LIBERO}
\label{subsec:libero}

Table~\ref{tab:libero_main} summarizes the results on LIBERO. MCF-Proto achieves the best overall performance with an average success rate of \(97.7\%\), outperforming all baselines. In particular, it achieves \(98.6\%\) on Spatial, \(98.8\%\) on Object, \(98.4\%\) on Goal, and \(95.0\%\) on Long.

Compared with the strongest prior baseline, OpenVLA-OFT, our method improves the average score from \(97.1\%\) to \(97.7\%\), with gains on Spatial, Goal, and Long. The gain on the Long suite is especially notable, since long-horizon tasks are more sensitive to compounding action errors. This suggests that expressing actions in a learned local frame and composing them from shared prototypes can improve action consistency over extended horizons.

More broadly, these results indicate that the improvement does not simply come from a stronger perceptual or sequence modeling backbone, since the backbone is unchanged and only the action head is modified. Instead, the gains suggest that action parameterization itself can materially affect policy performance. The consistent improvement across multiple LIBERO suites further supports the view that local geometric alignment and structured action composition provide a useful inductive bias for manipulation behaviors beyond a single task family.

\subsection{Main results on LIBERO-plus}
\label{subsec:liberoplus}
Table~\ref{tab:libero_plus_main} reports the results on LIBERO-plus. Our method performs best on six of the seven perturbation categories: Camera (\(69.7\%\)), Robot (\(66.0\%\)), Light (\(98.3\%\)), Background (\(93.3\%\)), Noise (\(84.0\%\)), and Layout (\(82.1\%\)). On Language, it achieves \(80.1\%\), staying close to the best reported baseline (\(81.5\%\)).

The robustness gains on LIBERO-plus are not uniform across perturbation types, and this pattern is informative. The largest improvements occur under Camera and Robot perturbations, where the mapping from perception to action is most affected by geometric changes in viewpoint or shifts in the robot's initial pose. This is precisely the setting where a fixed world-frame action parameterization is most brittle, since similar interaction intents may correspond to substantially different action vectors. In contrast, expressing actions in an observation-conditioned local frame allows the policy to absorb part of this variation before decoding the control command. Strong results under Light, Background, and Noise further suggest that the proposed action parameterization remains effective under perceptual disturbances. The smaller relative gain on Language perturbations is also consistent with our design: MCF-Proto primarily restructures the action output space rather than directly improving language understanding.
\begin{figure}[t]
  \centering
  \includegraphics[width=1.0\textwidth]{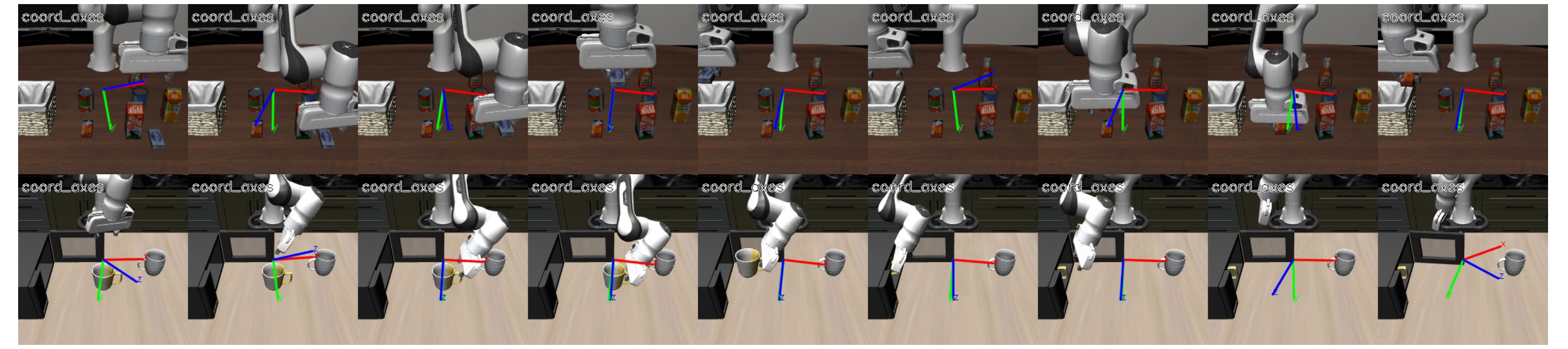}
  \caption{Learned MCF visualization on LIBERO. The predicted MCF provides a smooth local basis whose axes remain strongly compatible with the dominant motion direction in two representative manipulation tasks, despite being learned without explicit directional supervision.}
  \label{fig:libero}   
\end{figure}
\subsection{Local-frame action concentration analysis}
\label{subsec:concentration}
\begin{table}[h]
\centering
\caption{Quantitative comparison of action concentration metrics between world-frame and MCF}
\label{tab:action_concentration}
\small
\begin{tabular}{lccc}
\toprule
\textbf{Metric}               & \textbf{World (mean±std)} & \textbf{MCF (mean±std)} & \textbf{Relative Change} \\ \midrule
Covariance Trace              & 6.71 ± 1.93               & 5.55 ± 2.25               & -17.8\%                  \\
Average Pairwise Distance     & 3.41 ± 0.46               & 2.93 ± 0.53               & -14.0\%                  \\

PCA Top-3 Explained Variance  & 0.77 ± 0.03               & 0.90 ± 0.03               & +16.7\%                  \\
Effective Rank                & 4.03 ± 0.54               & 1.94 ± 0.21               & -51.3\%                  \\ \bottomrule
\end{tabular}%
\end{table}
To quantify how the learned local frame changes the structure of the action space, we compare world-frame actions and local-frame actions using four task-wise statistics: Covariance Trace, Average Pairwise Distance, Top-3 PCA Explained Variance, and Effective Rank. As shown in Table~\ref{tab:action_concentration}, the local-frame representation is consistently more concentrated than the world-frame representation across all four metrics. Specifically, local-frame actions have lower covariance trace (\(5.55\) vs.\ \(6.71\)) and lower average pairwise distance (\(2.93\) vs.\ \(3.41\)), while showing higher Top-3 PCA explained variance (\(0.90\) vs.\ \(0.77\)) and lower effective rank (\(1.94\) vs.\ \(4.03\)).

These results indicate that the learned motion-centric frame reorganizes action variability into a lower-dimensional and more concentrated local representation. Figure~\ref{fig:scatter} illustrates the same effect qualitatively: in the local frame, task-specific groupings become more distinguishable than in the world frame, suggesting that geometric alignment makes recurring manipulation patterns easier to separate and reuse.

This concentration effect is important because it provides a quantitative explanation for why prototype-based composition becomes more effective in the learned local frame. When action variation is distributed across fewer dominant directions, a small shared prototype dictionary can capture more of the task-relevant structure with less redundancy. In this sense, the MCF does not merely rotate the action space; it reorganizes it into a form that is more amenable to compact reuse.
\begin{figure}[t]
  \centering
  \resizebox{0.9\textwidth}{!}{%
  \includegraphics{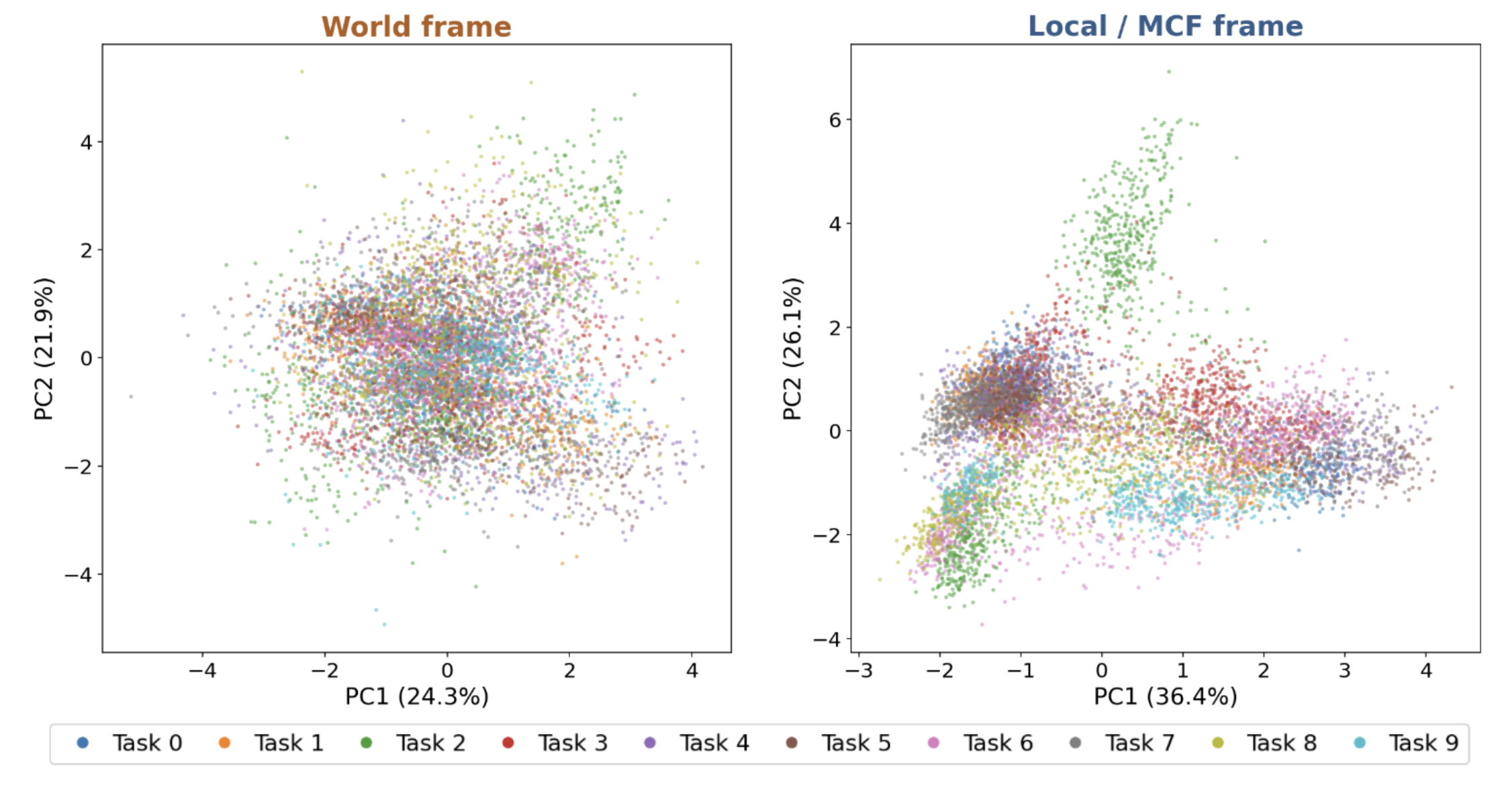}
  }
  \caption{Visualization of action distributions in the world frame and the learned local frame. Compared with world-frame actions, actions expressed in the Motion-Centric Frame are more concentrated and form more separable task-specific clusters, suggesting that the learned local representation better captures shared and reusable manipulation structure.}
  \label{fig:scatter}
  \vspace{-1pt}
\end{figure}
\subsection{Geometric diagnostic for action compatibility}
\label{subsec:compatibility}
\begin{table}[h]
\centering
\caption{Angular error (°) between MCF axes and demonstrated end-effector motion directions}
\small
\label{tab:rot_error_transpose}
\begin{tabular}{lccccc}
\toprule
\textbf{Metric} & \textbf{place} & \textbf{door-close} & \textbf{knob-turn} & \textbf{drawer-close} & \textbf{insert} \\
\midrule
MCF only & $35.4^\circ\pm3.1^\circ$ & $34.7^\circ\pm1.4^\circ$ & $35.7^\circ\pm2.9^\circ$ & $35.4^\circ\pm7.5^\circ$ & $32.8^\circ\pm5.9^\circ$ \\
MCF-Proto      & $18.8^\circ\pm3.9^\circ$ & $18.2^\circ\pm1.5^\circ$ & $10.8^\circ\pm3.4^\circ$ & $17.4^\circ\pm2.7^\circ$ & $13.2^\circ\pm6.8^\circ$ \\
\bottomrule
\end{tabular}
\end{table}
We further assess whether the learned motion-centric frame captures useful geometric structure by measuring how well its axes align with demonstrated local motion on representative LIBERO-Long tasks. For each time step, we compute the minimum unsigned angle between the demonstrated end-effector displacement direction and the three axes of the predicted frame. A smaller angle indicates that at least one learned axis is compatible with the demonstrated local motion.

Table~\ref{tab:rot_error_transpose} shows that MCF+Proto consistently achieves lower angular error than MCF only, reducing the discrepancy from roughly \(33^\circ \sim 36^\circ\) to \(11^\circ \sim 19^\circ\). This suggests that prototype-based local action composition not only improves prediction accuracy, but also encourages the learned frame to expose directions that are more compatible with demonstrated motion.

Together, these results suggest that the learned frame is not arbitrary: it develops axes that reflect meaningful task geometry. Moreover, the improvement from MCF only to MCF-Proto indicates that local prototype composition helps stabilize the emergence of this directional structure. A more structured local action space appears to make the frame itself easier to learn in a motion-aligned way.

\subsection{Real-world experiments on OpenArm}
\label{subsec:realworld}

\begin{table}[t]
\centering
\caption{Real-world evaluation on OpenArm. Task 1 reports the success rate (\%) of placing at least the first \(k\in\{1,2,3,4\}\) test tubes correctly, together with the average completed trajectory length. Tasks 2 and 3 report overall task success rate. Higher is better for all metrics.}
\label{tab:openarm_real}
\small
\setlength{\tabcolsep}{3.2pt} 
\begin{tabular}{lccccc|c|c}
\toprule
& \multicolumn{5}{c|}{\textbf{1: Test-tube placement}} & \textbf{2: Bi. transfer} & \textbf{3: Bowl stack} \\
\cmidrule(lr){2-6} \cmidrule(lr){7-7} \cmidrule(lr){8-8}
\textbf{Method} & \textbf{@1} & \textbf{@2} & \textbf{@3} & \textbf{@4} & \textbf{Avg. Len} & \textbf{Success} & \textbf{Success} \\
\midrule
\(\pi_{0}\)
& 64.0$\pm$4.2 & 26.0$\pm$3.7 & 14.0$\pm$2.9 & 8.0$\pm$2.3
& 1.12$\pm$0.07
& 32.0$\pm$6.8 & 66.0$\pm$7.3 \\

\(\pi_{0.5}\)
& 80.0$\pm$3.6 & 32.0$\pm$3.1 & 24.0$\pm$2.7 & 18.0$\pm$2.4
& 1.54$\pm$0.06
& 40.0$\pm$5.9 & 72.0$\pm$6.1 \\

Ours
& \textbf{84.0$\pm$3.1} & \textbf{56.0$\pm$2.8} & \textbf{42.0$\pm$2.6} & \textbf{38.0$\pm$2.2}
& \textbf{2.20$\pm$0.05}
& \textbf{64.0$\pm$4.7} & \textbf{82.0$\pm$4.3} \\
\bottomrule
\end{tabular}
\end{table}
We further evaluate our method on a real OpenArm platform\footnote{\url{https://openarm.dev}} with three long-horizon manipulation tasks: placing four test tubes onto a rack, bimanually transferring a block between bowls, and stacking two bowls. Table~\ref{tab:openarm_real} summarizes the quantitative results. Across all three tasks, our method shows stronger long-horizon reliability than the baseline variants.

These results suggest that the benefits of local action alignment and prototype-based composition also extend to real-world manipulation. The advantage is especially meaningful in multi-stage tasks, where small low-level prediction errors can accumulate over time and destabilize the full trajectory.
\begin{figure}[t]
  \centering
  \includegraphics[width=1.0\textwidth]{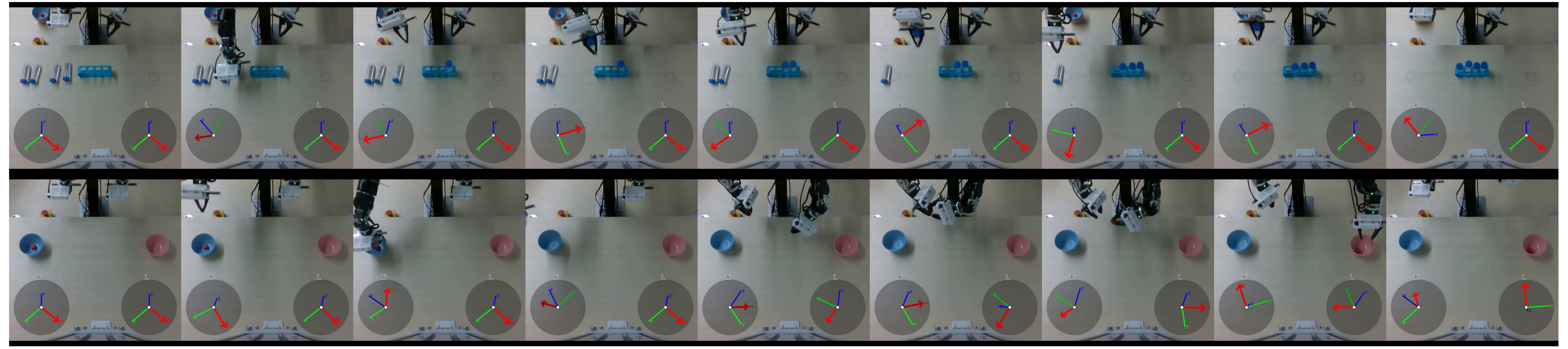}
  \caption{The predicted local frame maintains a stable local basis whose axes remain geometrically compatible with the principal task motion direction in real-world rollouts, suggesting that similar geometric structure also emerges in real-world manipulation.}
  \label{fig:openarm}   

\end{figure}
\subsection{Ablation studies}
\label{subsec:ablation}
\begin{table}[t]
\centering

\setlength{\tabcolsep}{5pt}
\caption{Hierarchical ablation of MCF-Proto. Avg denotes the average task success rate over all evaluation suites.}

\begin{tabular}{lccccc}
\toprule
\textbf{Method} & \textbf{MCF} & \textbf{Proto} & \(\boldsymbol{L_{\text{ortho}}}\) & \(\boldsymbol{L_{\text{smooth}}}\) & \textbf{Avg (\%)} \\
\midrule
BC w/ MLP head                & \xmark & \xmark & \xmark & \xmark & 93.8$\pm$0.6 \\
MCF only                      & \cmark & \xmark & \xmark & \xmark & 94.7$\pm$0.5 \\
World-Proto                   & \xmark & \cmark & \xmark & \xmark & 94.4$\pm$0.6 \\
MCF + Proto                   & \cmark & \cmark & \xmark & \xmark & 96.9$\pm$0.3 \\
MCF + Proto + \(L_{\text{ortho}}\)
                              & \cmark & \cmark & \cmark & \xmark & 97.4$\pm$0.2 \\
MCF-Proto (full)              & \cmark & \cmark & \cmark & \cmark & 97.7$\pm$0.3 \\
\bottomrule
\end{tabular}
\vspace{-2pt}
\label{tab:ablation_hierarchical}
\end{table}

Table~\ref{tab:ablation_hierarchical} shows hierarchical ablations of MCF-Proto. Introducing the motion-centric frame alone improves over the BC baseline (\(93.8\%\rightarrow94.7\%\)), and prototype composition alone in the world frame is also beneficial (\(94.4\%\)). Combining both yields a substantially larger gain (\(96.9\%\)), indicating that MCF and prototype composition are complementary. Adding \(\mathcal{L}_{\text{ortho}}\) further improves performance to \(97.4\%\), and adding the temporal smoothness regularizer \(\mathcal{L}_{\text{smooth}}\) gives the full model with the best average success rate (\(97.7\%\)).

Together, the diagnostic and ablation results suggest that MCF and prototype composition are not merely additive choices, but reinforce each other during training: a better local frame makes actions more comparable across states and tasks, enabling easier prototype reuse; in turn, prototype-based local prediction encourages the frame to expose directions more consistent with demonstrated motion. This mutual reinforcement likely explains why the combined model yields much larger gains than either component alone.

\section{Discussion and Limitations}

MCF-Proto shows that action parameterization is an important but underexplored design axis for vision-language-action policies. By redesigning only the action head, the method consistently improves LIBERO performance and LIBERO-plus robustness, with especially clear gains under camera and robot perturbations. Our analyses further suggest that these gains stem from a more structured action representation: the learned motion-centric frame concentrates action variation in a local coordinate system, while prototype-based composition yields more compact and reusable motion structure after alignment. These results indicate that lightweight geometric and compositional structure on the output side can provide a useful inductive bias for manipulation learning.

At the same time, the current study has several limitations. The evaluation is limited to LIBERO, LIBERO-plus, and a small number of real-world tasks, and broader validation is needed to assess generality across more diverse settings, especially for highly dynamic, deformable, or fine-grained contact-rich manipulation. In addition, the learned motion-centric frame remains an implicit latent representation rather than a semantically grounded coordinate system. More broadly, as a behavioral cloning method, MCF-Proto still depends on the coverage and quality of offline demonstrations. An important direction for future work is to test whether the same action-head design continues to help at larger scale, under online adaptation, and with more explicit geometric or object-aware representations.



{
    \small
    \bibliographystyle{plainnat}
    \bibliography{main}



}

\appendix

\section*{Anonymous Code Repository}
An anonymized code repository for reproducing the experiments is provided below:

\begin{center}
\url{https://anonymous.4open.science/r/MCF-Proto-455B}
\end{center}

\section{Why compact local representations favor directional alignment}
\label{app:theory}

In the main paper, MCF-Proto combines a learned local action frame with prototype-based action composition. Empirically, this joint design produces two closely related effects: actions become more concentrated when expressed in the learned local frame, and the frame itself develops stable axes that align strongly with demonstrated end-effector motion.

This appendix provides a simple theoretical view of that connection. We study an idealized setting in which the learning system prefers local coordinates that yield a more compact action representation. Under standard symmetry assumptions, the preferred frame aligns its axes with the principal directions of the action distribution. While this analysis does not model the full MCF-Proto architecture, it helps explain why prototype-based local action reuse can naturally induce motion-aligned coordinate structure.

Let \(a \in \mathbb{R}^d\) be a zero-mean action in the world frame, and let
\[
u = R^\top a
\]
denote its representation in a learned local frame \(R \in SO(d)\). Since reconstruction alone is invariant to orthogonal changes of basis, any preference over \(R\) must come from additional structure imposed on the local coordinates \(u\). Here we model that preference through the objective
\[
\min_{R \in SO(d)} J(R), \qquad
J(R) := \mathbb{E}\|R^\top a\|_1,
\]
which favors frames in which action variation is more concentrated across coordinates.

Assume \(a\) follows an elliptically symmetric distribution with covariance
\[
\Sigma = \mathbb{E}[aa^\top].
\]
Then \(a\) can be written as
\[
a = \Sigma^{1/2} z,
\]
where \(z\) is a spherically symmetric random vector with identity covariance.

\begin{proposition}
Under the above assumptions,
\[
\mathcal{J}(R)
=
c \sum_{i=1}^d \sqrt{(R^\top \Sigma R)_{ii}}
\]
for some constant \(c = \mathbb{E}|z_1| > 0\). Moreover, \(\mathcal{J}(R)\) is minimized when \(R^\top \Sigma R\) is diagonal. Equivalently, the minimizing frame \(R\) has columns given by eigenvectors of \(\Sigma\). Thus, up to permutation and sign flips, the learned coordinate axes align with the principal directions of the training action distribution.
\end{proposition}

\begin{proof}
Let \(r_i\) denote the \(i\)-th column of \(R\). Then the \(i\)-th local coordinate is
\[
u_i = e_i^\top R^\top a = r_i^\top a.
\]
Since \(a = \Sigma^{1/2} z\),
\[
u_i = r_i^\top \Sigma^{1/2} z.
\]
By spherical symmetry of \(z\), the distribution of \(r_i^\top \Sigma^{1/2} z\) depends only on its standard deviation. Hence
\[
\mathbb{E}|u_i|
=
c \sqrt{\mathrm{Var}(u_i)}
=
c \sqrt{r_i^\top \Sigma r_i}
=
c \sqrt{(R^\top \Sigma R)_{ii}},
\]
where \(c = \mathbb{E}|z_1|>0\) is independent of \(R\). Summing over coordinates gives
\[
\mathcal{J}(R)
=
\sum_{i=1}^d \mathbb{E}|u_i|
=
c \sum_{i=1}^d \sqrt{(R^\top \Sigma R)_{ii}}.
\]

Now let \(d(R)\) denote the diagonal vector of \(R^\top \Sigma R\). By the Schur--Horn theorem, \(d(R)\) is majorized by the eigenvalue vector \(\lambda(\Sigma)\). Since the function \(x \mapsto \sqrt{x}\) is concave on \(\mathbb{R}_+\), Karamata's inequality implies
\[
\sum_{i=1}^d \sqrt{(R^\top \Sigma R)_{ii}}
\ge
\sum_{i=1}^d \sqrt{\lambda_i},
\]
with equality when \(R^\top \Sigma R\) is diagonal.

Finally, if \(R^\top \Sigma R = \Lambda\) for some diagonal matrix \(\Lambda\), then
\[
\Sigma R = R \Lambda.
\]
Therefore, each column \(r_i\) of \(R\) satisfies
\[
\Sigma r_i = \Lambda_{ii} r_i,
\]
so \(r_i\) is an eigenvector of \(\Sigma\). Hence the minimizing frame aligns its axes with the eigen-directions, i.e., the principal directions, of the training action distribution. The solution is unique only up to permutation of axes and sign flips, and is additionally non-unique within eigenspaces corresponding to repeated eigenvalues.
\end{proof}

\section{Experimental Details}
We evaluate on the LIBERO and LIBERO-plus benchmark, including LIBERO-Spatial, LIBERO-Object, LIBERO-Goal, and LIBERO-long. Following the standard protocol, training uses the full LIBERO mixture without a held-out validation split. Input images are resized to \(224 \times 224\), and proprioceptive inputs use 6 active state dimensions from the 8-dimensional robot state. Each task is evaluated over 50 rollouts in LIBERO and 1 rollout in LIBERO-plus, and success rate is reported.

MCF-Proto uses Qwen3-VL-4B-Instruct as the vision-language backbone and a MCF-Proto action head. The action space is 7-dimensional, and the model predicts a horizon of 7 actions. We use latent dimension \(d=3\), with \(K_t=16\) translation prototypes and  \(K_r=16\) rotation prototypes. 

Training uses AdamW with cosine decay and 5k-step warmup. The base learning rate is \(2.5\times10^{-5}\). Models are trained for 40k steps on 4 NVIDIA A-100 GPUs with effective batch size 128. The total training time cost about 36 hours. At inference, we use the 40k-step checkpoint and execute chunked action prediction with chunk size 8.

\section{Prototype Usage Analysis}
\label{Prototype Usage Analysis}
\begin{figure}[t]
  \centering
  \includegraphics[width=1.0\textwidth]{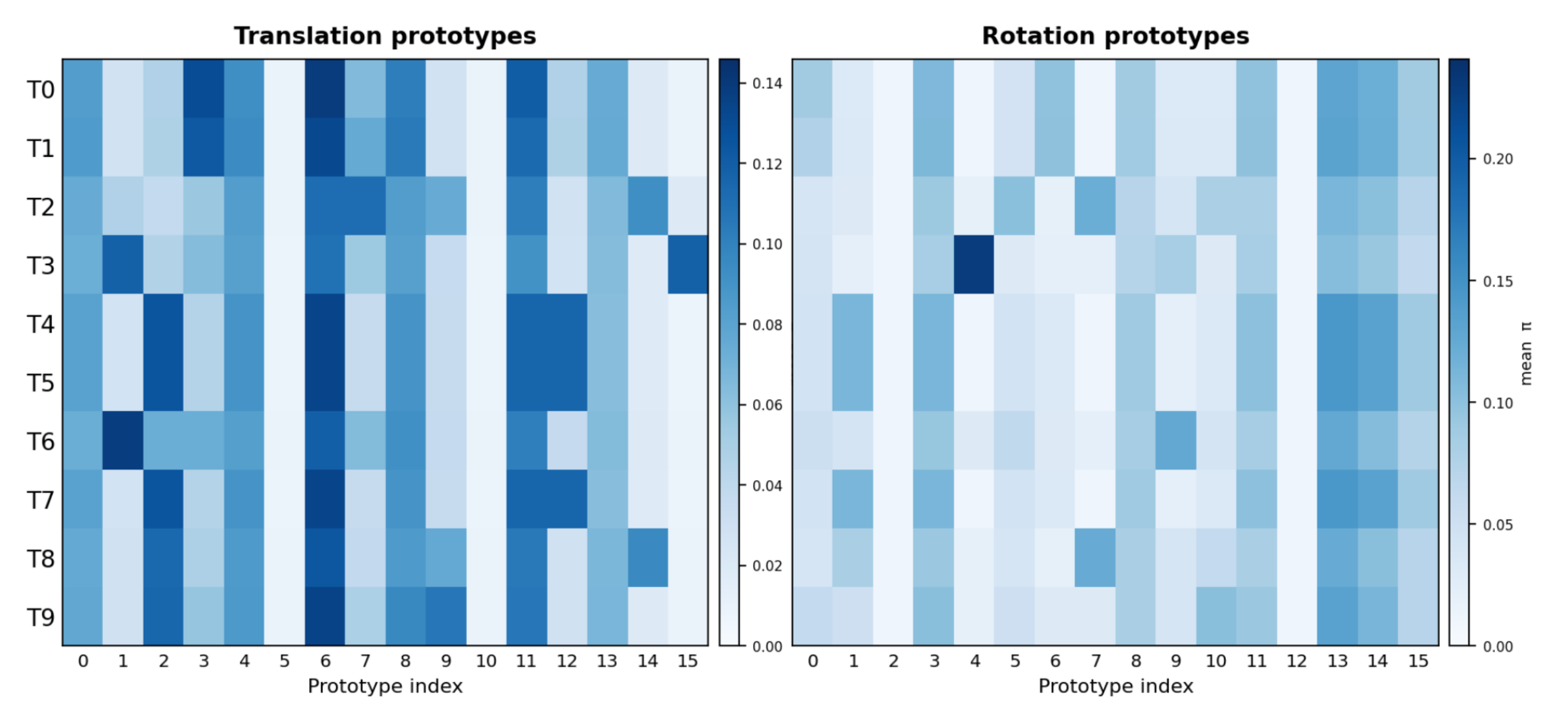}
  \caption{Prototype usage distribution of the learned action model across LIBERO-long tasks (K = 16)}
  \label{fig:proto}   

\end{figure}
To understand how the learned action vocabulary is utilized across tasks, we analyze the prototype gating distribution $\pi$ of the trained model on the LIBERO benchmark. For each task, we collect per-frame gating weights from all episodes in the training split and compute the mean $\pi$ over frames, yielding a tasks × K matrix for both translation and rotation prototypes (K = 16). The resulting heatmaps are shown in Figure~\ref{fig:proto}.
Two complementary patterns are visible. First, a set of prototypes receives consistently high activation across all ten tasks, forming a shared backbone of task-agnostic motion primitives. These broadly-used prototypes correspond to elementary manipulation phases—reaching toward an object, grasping, lifting, and returning—that appear in virtually every task regardless of scene geometry or object type. Second, a distinct subset of prototypes exhibits strong task-group specificity. In the translation panel, tasks T4, T5, and T7 share elevated weights on a common prototype associated with downward placement trajectories, while T6 and T3 preferentially activate a separate prototype reflecting approach motions toward the stove region. The rotation panel reveals an even sharper specialization: T3 (turn knob) concentrates nearly a quarter of its rotational gating mass onto a single prototype, reflecting the large, structured wrist rotation unique to that task—a pattern absent in all other tasks.
Taken together, these results support the hypothesis that the MCF-Proto action head learns a compositional action vocabulary: a compact set of reusable primitives covers universal motion structure, while a sparse collection of specialized prototypes captures task-conditioned variation. This decomposition emerges purely from task-level supervision without any explicit prototype assignment, demonstrating that structured action representations can be self-organized through the proposed gating mechanism.

\section{MCF Axis Alignment Analysis}
\begin{figure}[t]
  \centering
  \includegraphics[width=1.0\textwidth]{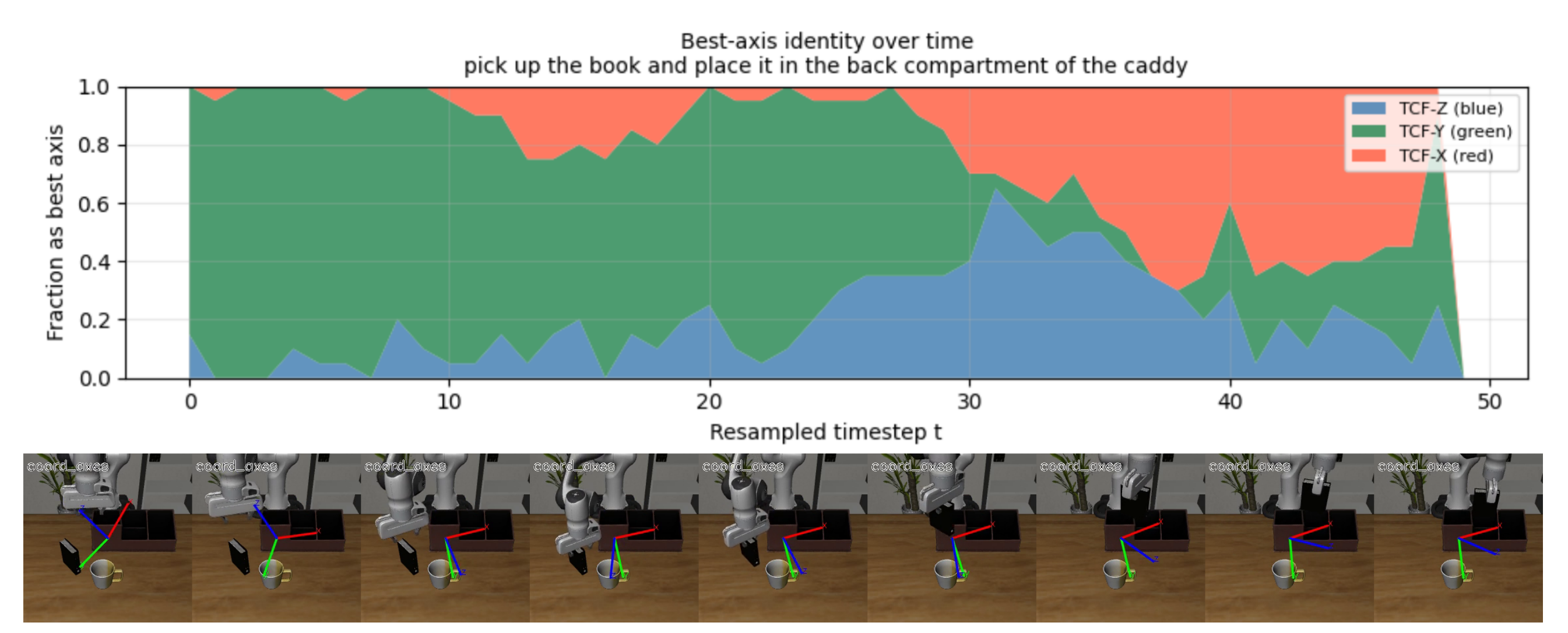}
  \caption{Dominant MCF axis over time for task "pick up the book and place it in the black compartment of the caddy".}
  \label{fig:mainaxis}   

\end{figure}
\label{MCF Axis Alignment Analysis}
To verify that the MCF captures semantically meaningful action structure, we examine how the dominant action axis evolves over the course of a task. Figure \ref{fig:mainaxis} shows the per-timestep distribution of the "best axis"---the MCF axis carrying the largest action magnitude---aggregated across all episodes of the task \textit{pick up the book and place it in the black compartment of the caddy}.

The distribution reveals a clear phase structure aligned with the task's semantic stages. During the early phase (\(t \approx 0\text{–}25\)), actions are predominantly decomposed along the MCF-Y (green) axis, corresponding to downward reaching toward the book. In the grasping and posture adjustment phase (\(t \approx 25\text{–}35\)), the MCF-Z (blue) axis dominates as the robotic arm rotates about its axis after securing the book. During the lateral translation phase (\(t \approx 35\text{–}45\)), the MCF-X (red) axis governs the motion of moving the book toward the caddy. In the final phase (\(t \approx 45\text{–}50\)), the MCF-Y (green) axis again dominates, corresponding to the book-placement movement.
An overlap between Phase 2 and Phase 3 indicates that the robotic arm performs simultaneous translation and rotation, which is also consistent with empirical observations.

Crucially, this axis-switching behavior emerges without any explicit phase annotation or axis supervision. Rather than indicating that the frame rotates to follow the instantaneous motion direction at every step, this pattern suggests that the MCF evolves smoothly over time, while stage changes are primarily expressed through shifts in which axis becomes dominant. This interpretation is consistent with the smoothness regularization described in the main text, which encourages temporal consistency in the frame and explains multi-stage changes through axis reassignment rather than excessive frame rotation. Overall, the smooth yet structured transitions between dominant axes suggest that the MCF provides a principled decomposition of multi-stage manipulation into a sequence of locally near-uniaxial motions, supporting more compact and reusable action representations.

\section{Additional visualization}
\begin{figure}[ht]
  \centering
  \includegraphics[width=1.0\textwidth]{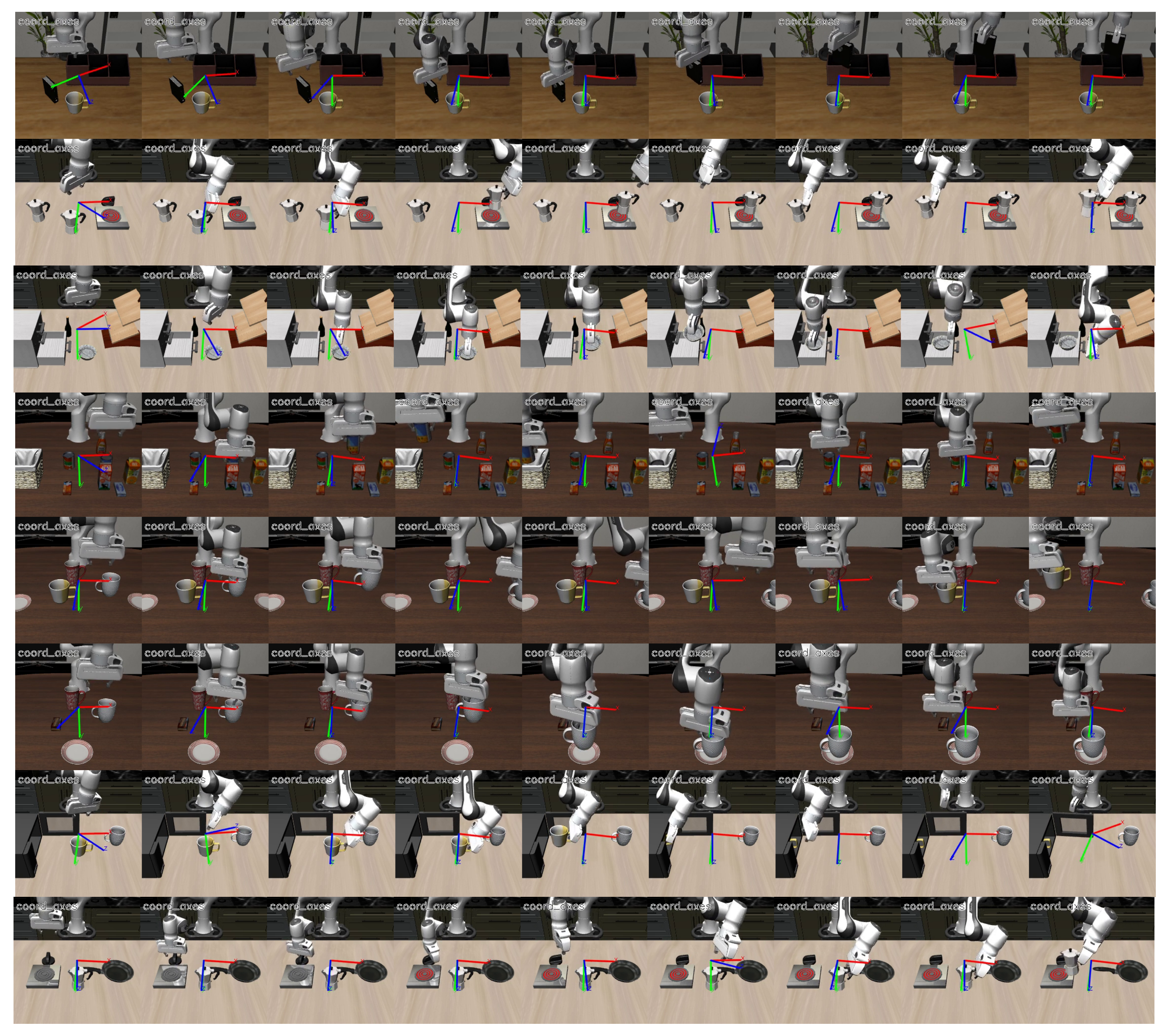}
  \caption{Additional visualization results of MCF-Proto
}
  \label{fig:addvis}   
\end{figure}


\newpage

\end{document}